\documentclass[11pt,a4paper]{article}

\usepackage[margin=1in]{geometry}
\usepackage{amsmath}
\usepackage{amssymb}
\usepackage{amsthm}
\usepackage{mathtools}
\usepackage{bm}
\usepackage{graphicx}
\usepackage{booktabs}
\usepackage{array}
\usepackage{multirow}
\usepackage{microtype}
\usepackage{enumitem}
\usepackage{algorithm}
\usepackage{algpseudocode}
\usepackage{float}
\usepackage{xcolor}
\usepackage{makecell}
\usepackage{authblk}

\usepackage{hyperref}
\hypersetup{
  colorlinks  = true,
  linkcolor   = blue!70!black,
  citecolor   = green!50!black,
  urlcolor    = magenta!70!black,
  pdfauthor   = {Silvano Coletti, Francesca Fallucchi},
  pdftitle    = {Graph Vector Field: A Unified Framework for
                 Multimodal Health Risk Assessment},
  pdfsubject  = {Discrete differential geometry, simplicial complexes,
                 multimodal health risk, mixture-of-experts},
  pdfkeywords = {GVF, Helmholtz-Hodge decomposition, simplicial complex,
                 discrete exterior calculus, mixture-of-experts,
                 wearable health monitoring}
}
\usepackage{cleveref}

\theoremstyle{plain}
\newtheorem{theorem}{Theorem}[section]
\newtheorem{proposition}[theorem]{Proposition}

\theoremstyle{definition}
\newtheorem{definition}[theorem]{Definition}

\theoremstyle{remark}
\newtheorem{remark}[theorem]{Remark}

\newcommand{\R}{\mathbb{R}}
\newcommand{\K}{\mathcal{K}}
\newcommand{\eps}{\varepsilon}

\newcommand{\gradK}{\nabla_{\!\K}}
\newcommand{\divK}{\mathrm{div}_{\K}}
\newcommand{\curlK}{\mathrm{curl}_{\K}}
\newcommand{\bundleE}{\varepsilon(\K)}
\newcommand{\sect}{\Gamma}
\newcommand{\DPS}{\mathrm{DPS}}
\newcommand{\CRI}{\mathrm{CRI}}
\newcommand{\dGH}{d_{\mathrm{GH}}}
\newcommand{\phys}{\mathrm{phys}}
\newcommand{\behav}{\mathrm{beh}}
\newcommand{\env}{\mathrm{env}}
\newcommand{\ext}{\mathrm{ext}}

\begin{document}

%% --- Title --------------------------------------------------
\title{Graph Vector Field: A Unified Framework for
       Multimodal Health Risk Assessment from Heterogeneous
       Wearable and Environmental Data Streams}

%% --- Authors (authblk) --------------------------------------
\author[1,2]{Silvano Coletti}
\author[1]{Francesca Fallucchi}
\affil[1]{Universit\`{a} degli Studi Guglielmo Marconi,
  Roma, Italy}
\affil[2]{Chelonia SA, Lugano, Switzerland}

\date{\today}

\maketitle

%% --- Abstract -----------------------------------------------
\begin{abstract}
Digital health research has advanced dynamic graph-based disease
models, topological learning on simplicial complexes, and multimodal mixture-of-experts
architectures, but these strands remain largely
disconnected and do not yet support mechanistically interpretable
multimodal risk dynamics.
We propose a \textbf{Graph Vector Field (GVF)} framework that
models health risk as a vector-valued field evolving on
time-varying graphs and simplicial complexes, and that couples
discrete differential-geometric operators with modality-structured
mixture-of-experts.
Individuals, physiological subsystems, and their higher-order
interactions form the cells of an underlying complex; risk is
represented as a vector-valued cochain whose temporal evolution is
parameterised with Hodge Laplacians and discrete exterior calculus
operators.
This yields a Helmholtz--Hodge-type decomposition of risk dynamics
into potential-driven (exact), circulation-like (coexact), and
topologically constrained (harmonic) components, providing a direct
link between model terms and interpretable mechanisms of
propagation, cyclic behaviour, and persistent modes.
Multimodal inputs from wearable sensors, behavioural and
environmental context, and clinical or genomic data are
incorporated through a bundle-/sheaf-inspired mixture-of-experts
architecture, in which modality-specific latent spaces are attached
as fibres to the base complex and coupled via learned consistency
maps.
This construction separates modality-specific from shared
contributions to risk evolution and offers a principled route
toward modality-level identifiability.
GVF integrates geometric dynamical systems,
higher-order topology (enforced indirectly via geometric
regularisation and Hodge decomposition rather than through
the expert forward pass), and structured multimodal fusion
into a single framework for interpretable, modality-resolved
risk modelling in longitudinal cohorts.
The present paper develops the mathematical foundations,
architectural design, and formal guarantees of the framework;
empirical validation on real-world cohorts is the subject of
ongoing work.
\end{abstract}

\smallskip
\noindent\textbf{Keywords:}
Helmholtz--Hodge decomposition,
simplicial complex,
discrete exterior calculus,
mixture-of-experts,
multimodal health risk,
wearable sensing,
vector bundle

\smallskip
\noindent\textbf{MSC (2020):}
Primary 68T07;
Secondary 55N31, 58A15, 62R40, 92C50

%% ============================================================
\section{Introduction}
\label{sec:intro}
%% ============================================================

The proliferation of wearable devices and connected sensors is
transforming how health can be monitored and modelled.
Consumer- and medical-grade wearables now act, \textit{de facto}, as
pervasive antennas that continuously transmit high-frequency
physiological and behavioural signals, and are increasingly
complemented by electronic health records (EHR), environmental
exposures, and genomic information in large longitudinal
cohorts~\cite{Wang2025allofus,Zhang2026}.
Together, these data streams promise a shift from episodic assessment
to continuous, multimodal characterisation of individual health
trajectories and risk.

Current artificial intelligence (AI) methods already exploit parts of
this ecosystem: spatiotemporal graph neural networks on wearable-derived
skeletons and sensor networks for injury risk and health
warnings~\cite{Xie2025,Li2025,Fang2025}; graph-based models for
long-term stress, sleep, and resilience
monitoring~\cite{Wang2023,Biswas2026}; dynamic diffusion processes on
comorbidity graphs learned from EHR~\cite{Qian2020}; geometric
regression on data-driven simplicial complexes using Hodge
Laplacians~\cite{Gajer2025}; and large-scale multimodal fusion or
mixture-of-experts (MoE) architectures that integrate EHR, wearables,
and imaging~\cite{Wang2025allofus,Zhang2026,Wang2025moe}.
However, these advances remain largely fragmented.
Most models treat risk as an unstructured scalar output or as an
abstract latent representation in Euclidean space, provide limited
mechanistic insight into how risk propagates over physiological,
behavioural, and environmental structures, and do not explicitly
disentangle the contributions of heterogeneous modalities.

This limitation becomes evident in scenarios where individuals share
similar scalar risk scores but differ profoundly in how risk emerges
and propagates through their surrounding networks.
For one person, elevated risk may be driven by a directed flow from a
specific environmental exposure, suggesting source suppression as a
primary intervention.
For another, comparable aggregate risk may arise from a
self-reinforcing cycle involving wearables, behaviour, and context,
where no single source can be isolated and cycle-breaking interventions
are required instead.
Under scalar or unstructured latent models these cases are effectively
indistinguishable, even though they are clinically and mechanistically
distinct.

GVF is designed to address these gaps.
It models health risk as a vector-valued field evolving on
time-varying simplicial complexes, uses discrete
differential-geometric operators (including Hodge--Laplacians) to
structure and interpret the dynamics of this field, and embeds
multimodal inputs in a bundle-structured mixture-of-experts
architecture that separates modality-specific from shared effects.
Our contributions are threefold: we provide a geometric
and topological substrate for dynamic risk fields that extends
graph-based diffusion models to higher-order domains; we define a
Hodge-theoretic decomposition of learned risk flows into components
with distinct clinical interpretations, naturally linked to different
classes of input modality; and we propose a differentiable,
bundle-structured expert architecture that jointly learns risk vectors
and flow decomposition over a data-driven simplicial complex
(constructed via threshold-based filtration guided by persistent
homology, not learned end-to-end) while explicitly organising
modality-specific and shared contributions.

%% ============================================================
\section{Evolution of AI for Multimodal Health Risk Modelling}
\label{sec:related}
%% ============================================================

This section reviews recent advances in AI for analysing multimodal
health data, with an emphasis on (i)~wearable-centric risk prediction,
(ii)~dynamic graph models of disease progression,
(iii)~geometric and topological learning, and
(iv)~multimodal fusion and mixture-of-experts architectures,
focusing on how these strands relate to modelling risk dynamics,
geometric and topological structure, and modality-aware
interpretability.

\subsection{Graph- and sequence-based models for wearable-driven risk}

A first line of work uses wearable sensors as the primary data source
and combines temporal models with graph architectures.
Spatio-temporal graph neural networks (ST-GNNs) have been applied to
athletic injury risk prediction by mapping multimodal wearable
signals---IMU, sEMG, PPG---onto a dynamic human skeleton graph, where
nodes represent joints and edges encode biomechanical
relationships~\cite{Xie2025}.
The ST-GNN captures coupled spatiotemporal patterns in joint
kinematics, and federated/meta-learning supports personalised
adaptation~\cite{Xie2025}.
Similarly, hybrid LSTM--GNN models over flexible sensor networks have
been proposed for real-time health warnings in college sports,
integrating multiple physiological time series in a lightweight,
edge-computing setting~\cite{Li2025}.
Injury risk prediction frameworks that transform temporal training data
into image-like encodings (e.g., Gramian Angular Fields, Markov
Transition Fields) and then to graphs also exploit parallel graph and
temporal convolutions to capture spatiotemporal structure, supplemented
by attention mechanisms to highlight important athlete interactions and
training-load features~\cite{Fang2025}.
Beyond sports, personalised graph-attention architectures have been used
to track long-term changes in maternal sleep and stress from smartwatch
and mobile-application data, producing individualised abnormality scores
over time~\cite{Wang2023}.
For mental health, participant-specific graphs where physiological
modalities---heart activity, electrodermal activity, motion---are nodes
and edges encode inter-signal dependencies have been used with GNNs to
classify resilience and stress states, revealing physiological
differences between resilience groups~\cite{Biswas2026}.
GNNs have been applied more broadly to EHR data~\cite{Choi2017}, brain
connectivity~\cite{Ktena2018}, and epidemiological
networks~\cite{Mossel2012}.

Collectively, these works demonstrate the value of graph-based and
sequence-based models for wearable-driven risk and state
monitoring~\cite{Xie2025,Li2025,Wang2023,Biswas2026,Fang2025}.
However, they generally implement dynamics implicitly through recurrent
or spatiotemporal neural architectures rather than via explicit operators
on a structured domain, treat risk as a scalar output or unstructured
latent representation, and do not exploit higher-order topological
structure or differential-geometric operators for mechanistic
interpretation.
Dynamic GNNs~\cite{Pareja2020,Xu2020,Li2018} handle time-varying graphs
but produce scalar outputs on pairwise (1-skeleton) graphs, where the
curl operator vanishes identically and cyclic risk dynamics are invisible
by construction.

\subsection{Dynamic graphs for disease progression}

A complementary line of research focuses on disease progression as a
dynamic process on clinical graphs.
Deep Diffusion Processes (DDP) model personalised, time-varying
comorbidity networks from EHR event data, where diseases are nodes and
dynamically weighted edges parameterise the intensity functions of a
multidimensional point process over disease-onset
times~\cite{Qian2020}.
Edge weights are modulated over time by a neural network that maps each
patient's history to dynamic influence factors, enabling a decomposition
of interaction effects into static pairwise interactions, temporal
influence, and dynamic components~\cite{Qian2020}.
This yields per-disease, time-varying intensity trajectories that can be
interpreted as a diffusion-like risk field over a learned, dynamic
comorbidity graph.

DDP thus provides a clear precedent for framing clinical risk and
progression as a dynamic process on a graph with interpretable
components~\cite{Qian2020}.
At the same time, it operates on a disease--disease graph derived from
EHR alone, without incorporating wearable, behavioural, or environmental
modalities, and without invoking higher-order topology, Hodge-theoretic
operators, or explicit geometric decomposition of risk dynamics.

\subsection{Geometric and topological learning in health}

Geometric and topological machine learning have begun to appear in
health-related prediction tasks, most notably through the use of
simplicial complexes and Hodge Laplacians.
The TDA paradigm~\cite{Carlsson2009,Ghrist2008} and persistent
homology~\cite{Edelsbrunner2002,Zomorodian2005,Edelsbrunner2010}
extract topological invariants from data.
DEC~\cite{Hirani2003,Desbrun2003} and Hodge theory on
graphs~\cite{Lim2020,Schaub2020} underpin flow decomposition; HHD
on graphs~\cite{Bhatia2013,Tong2003} has been applied to gene
expression and traffic flows~\cite{Barbarossa2020,Jiang2011}.
Simplicial neural networks~\cite{Bodnar2021,Xu2019} extend message
passing to higher-order complexes but produce scalar outputs.

Adaptive Geometric Regression (AGR) constructs a data-driven simplicial
complex (e.g., a $k$-NN nerve) over high-dimensional feature spaces,
equips it with a learned Riemannian structure and densities on
$p$-simplices, and applies Hodge Laplacians $L_p$ to perform heat
diffusion on $p$-cochains~\cite{Gajer2025}.
Spectral filtering of responses and a response-aware geometry that
down-weights edges spanning large response differences yield
diffusion-based features for downstream regression or classification,
with motivating examples in microbiome-based prediction of health
outcomes such as preterm birth~\cite{Gajer2025}.
AGR is, to our knowledge within this literature set, the most explicit
use of simplicial complexes and Hodge operators in a biomedical
context~\cite{Gajer2025}.
Nonetheless, diffusion is employed as an algorithmic regularisation and
feature-construction mechanism rather than as a model of time-evolving
health risk, and the setting is effectively single-modality---microbiome
plus covariates---without multimodal wearable, behavioural, or
environmental integration, or dynamical decomposition of risk via
Hodge-theoretic components.

\subsection{Multimodal fusion, foundation models, and mixture-of-experts}

A third strand addresses multimodal fusion across EHR, wearables,
imaging, and other data sources, with increasing emphasis on scalability
and robustness.
Large-scale pipelines that combine EHR and consumer wearable data for
incident disease prediction have shown that wearable augmentation can
yield consistent, statistically significant gains over EHR-only
baselines across a range of conditions~\cite{Wang2025allofus}.
These systems typically represent EHR and wearable streams via
time-aggregated features or foundation-model embeddings and fuse them
through concatenation, weighted concatenation, or feature
selection~\cite{Wang2025allofus}.
More recent work proposes multimodal foundation models that treat EHR
events and wearable time series as observations of a shared
continuous-time latent process, using modality-specific encoders and
cross-modal pretraining objectives to align dense physiological signals
with sparse clinical events~\cite{Zhang2026}.
This design improves physiological-state estimation, clinical event
forecasting, and time-to-event risk prediction, and addresses challenges
such as asynchronous sampling and modality-specific
missingness~\cite{Zhang2026}.

In parallel, mixture-of-experts (MoE) architectures have been introduced
for robust multimodal clinical prediction.
MoE-Health employs modality-specific encoders (e.g., for EHR and chest
radiography), learns modality-specific missingness embeddings, and routes
fused representations through experts specialised to different modality
subsets via a dynamic gating mechanism~\cite{Wang2025moe}.
This structure improves performance and robustness under varying patterns
of missing modalities~\cite{Wang2025moe}.
These multimodal approaches establish that fusing wearables with EHR
improves predictive performance at
scale~\cite{Wang2025allofus,Zhang2026} and that modality-structured
architectures such as MoE can enhance robustness to incomplete
data~\cite{Wang2025moe}.
However, they generally operate in unstructured Euclidean latent spaces,
without an explicit geometric or topological domain for risk, without
Hodge/DEC-based operators, and without a formal sheaf/bundle perspective
or identifiability analysis that would separate and interpret
modality-specific contributions over a structured base space.

Neural Sheaf Diffusion~\cite{Bodnar2022sheaf}, Bundle Neural
Networks~\cite{Bamberger2025}, and Bundle Morphism
Networks~\cite{Coda2022} introduce geometric structure into GNNs via
sheaf or bundle assignments over graph edges, but operate on pairwise
graphs (trivial curl) and use bundle structure for equivariance or
expressivity rather than modality identifiability; none applies
Helmholtz-Hodge decomposition to the learned field.

Across these four strands, contemporary AI for multimodal health risk
has progressed from spatiotemporal GNNs on wearable-derived
graphs~\cite{Xie2025,Li2025,Wang2023,Biswas2026,Fang2025}, through
dynamic diffusion-like processes on comorbidity
networks~\cite{Qian2020}, Hodge-Laplacian-based geometric regression on
simplicial complexes~\cite{Gajer2025}, to large-scale multimodal fusion
via foundation models and mixture-of-experts~\cite{Wang2025allofus,
Zhang2026,Wang2025moe}.
Yet no existing work jointly (i)~models multimodal health risk as a
vector-valued field on time-varying graphs or simplicial complexes,
(ii)~uses discrete differential-geometric operators to decompose and
interpret the dynamics of that field, and (iii)~embeds these dynamics in
a bundle-/sheaf-structured mixture-of-experts architecture aimed at
modality-level separation and identifiability.
This is the gap that GVF is designed to address, as developed in the
sections that follow.

%% ============================================================
\section{Mathematical Substrate}
\label{sec:substrate}
%% ============================================================

\subsection{From Pairwise Graph to Simplicial Complex}
\label{sec:complex}

Standard dynamic graph representations of sensor networks capture
pairwise interactions but fail to represent fundamentally higher-order
relationships: compound exposure events involving an agent, a spatial
cell, and an environmental sensor simultaneously; triadic
physiological interactions with non-additive effects; or clinical
co-morbidity clusters requiring three-entity modelling.
These are \emph{2-simplices}, not edges, and their inclusion is
necessary for the curl operator (\Cref{def:operators}) to be
non-trivially defined and for the HHD (\Cref{thm:hhd}) to capture
cyclic risk dynamics.

\begin{definition}[Multimodal Interaction Simplicial Complex $\K(t)$]
\label{def:complex}
Let $\K(t)$ be an abstract simplicial complex~\cite{Hatcher2002,
Munkres1984} over the vertex set
$V(t) = V^A(t)\cup V^S\cup V^E(t)\cup V^X$, where:
\begin{itemize}[leftmargin=1.5em, itemsep=2pt]
  \item $V^A(t)$: mobile agents bearing wearable devices;
  \item $V^S$: static spatial cells (geographic tessellation);
  \item $V^E(t)$: environmental sensor nodes (air quality,
        temperature, light, noise);
  \item $V^X$: optional external data nodes (EHR systems, genomic
        databases, surveillance feeds).
\end{itemize}
Simplices are typed by interaction order:
\begin{itemize}[leftmargin=1.5em, itemsep=2pt]
  \item \emph{1-simplices} $\{u,v\}$: pairwise interactions
        (proximity, co-location, agent-sensor exposure);
  \item \emph{2-simplices} $\{u,v,w\}$: simultaneous triadic
        interactions (compound exposure, tripartite co-location);
  \item \emph{3-simplices}: quaternary high-order compound events.
\end{itemize}
The 1-skeleton of $\K(t)$ reduces to a standard dynamic graph,
recovering prior GNN-based approaches as special cases.
\end{definition}

\subsection{Discrete Exterior Calculus Operators}
\label{def:operators}

We equip $\K(t)$ with the standard DEC structure of chain groups,
boundary operators, and their dual coboundary operators;
see~\cite{Hirani2003,Lim2020} for the full algebraic development.

\paragraph{Inner product and Hodge star}
We use the \textbf{standard combinatorial inner product} on
$k$-cochains: for $\alpha,\beta\in C^k(\K;\R)$,
\begin{equation}
  \langle\alpha,\beta\rangle_k
  = \sum_{\sigma^k\in\K} \alpha(\sigma^k)\,\beta(\sigma^k)
  \label{eq:inner}
\end{equation}
i.e.\ the identity Hodge star $\star = I$ (unweighted DEC),
the standard choice in discrete Hodge theory~\cite{Jiang2011,Lim2020}.
Under this inner product the gradient and divergence are mutually
adjoint, the HHD components are $L^2$-orthogonal (\Cref{thm:hhd}),
and the operator norm satisfies $\|\curlK\|_{C^1\to C^2}
= \|B_2^\top\|_2 \leq \sqrt{d_{\max}}$ where $d_{\max}$
is the maximum number of 2-simplices containing any edge
(relevant for the $\mathcal{L}_{\mathrm{geo}}$ bound in
\Cref{sec:training}).

We assume the standard structure of unweighted Discrete Exterior
Calculus (DEC)~\cite{Hirani2003,Lim2020}.
Specifically, we define the discrete gradient
$\gradK = B_1^\top\otimes I_m$, the discrete divergence
$\divK = B_1\otimes I_m$, and the discrete curl
$\curlK = B_2^\top\otimes I_m$,
where $B_1$ and $B_2$ are the signed node-edge and edge-face
incidence matrices of the complex.
The Hodge Laplacians are $\Delta_0=B_1 B_1^\top\otimes I_m$ and
$\Delta_1=B_1^\top B_1\otimes I_m + B_2 B_2^\top\otimes I_m$
(standard spectral theory~\cite{Chung1997,Lim2020,Schaub2020});
the identity $\curlK\circ\gradK=0$ follows from
$B_1 B_2=0$~\cite{Hirani2003}.

\begin{remark}
\label{rem:curl}
On a pure graph (1-skeleton), $\curlK\equiv\mathbf{0}$ structurally: no triangular faces exist on which circulation is
defined.
The curl operator is non-trivially defined only when $\K(t)$
contains 2-simplices.
This is the structural---not merely notational---necessity for the
cyclic risk detection capability of GVF.
\end{remark}

\subsection{Constructing Simplicial Complexes from Multimodal Health Data}
\label{sec:construction}

Translating heterogeneous clinical and sensor streams into a
mathematically rigorous simplicial complex $\K(t)$ requires
mapping physical and biological interactions to geometric structures.
The three simplex orders correspond to qualitatively distinct
interaction regimes.

\paragraph{0-simplices (nodes)}
Represent atomic entities: a patient wearing a smartwatch
($v_i\in V^A$), a stationary air quality monitor in a hospital
room ($v_e\in V^E$), or a static Electronic Health Record entry
($v_x\in V^X$).

\paragraph{1-simplices (edges)}
Capture pairwise, first-order interactions.
An edge $\{v_i,v_j\}$ may represent spatial proximity measured
via Bluetooth RSSI between two patients, while $\{v_i,v_e\}$
represents a patient's direct exposure to an environmental sensor.

\paragraph{2-simplices (triangles) and biological justification}
Standard pairwise graphs fail to capture irreducible triadic
interactions, which are pervasive in health data.
Consider the ``Gene-Environment-Phenotype'' interaction: if patient
$v_i$ carries a specific genomic vulnerability ($v_x$), occupies a
high-stress ward environment ($v_e$), and interacts with an
infectious vector ($v_j$), the resulting risk is often
\emph{non-additive}.
The 2-simplex $\{v_i,v_e,v_x\}$ geometrically encodes this
higher-order compounding effect.
Only by closing this triad can the DEC curl operator
(\Cref{rem:curl}) non-trivially compute circulation, effectively
modelling a self-reinforcing biological or environmental feedback
loop that cannot be deduced by observing the pairwise edges in
isolation.

\paragraph{Practical construction from raw streams}
Translating raw, asynchronous sensor streams into a mathematically
rigorous complex $\K(t)$ requires a formal pipeline comprising
temporal windowing, heterogeneous edge formation, and multimodal
simplex inclusion.

\textbf{(1) Temporal windowing.}
Raw data $S(t)$ are not evaluated instantaneously but aggregated
over a discrete sliding window of duration $\Delta t$
(e.g.\ $\Delta t = 5$ minutes for physiological and environmental
streams, or 24 hours for behavioural streams).
The complex $\K(t)$ is static within the interval
$[t,\,t+\Delta t)$ and updates discontinuously, ensuring the DEC
operators remain well-defined across stable epochs.

\textbf{(2) Heterogeneous edge formation.}
Because multimodal nodes do not inhabit a uniform metric space,
standard spatial distance functions are insufficient.
We define modality-specific adjacency thresholds $\tau$:
\begin{itemize}[leftmargin=1.5em, itemsep=2pt]
  \item \emph{Agent--agent edges:} $\{v_i^A, v_j^A\}$ is formed if
        spatial proximity (median Bluetooth RSSI over $\Delta t$)
        exceeds $\tau_{\mathrm{prox}}$, or if physiological synchrony
        (e.g.\ dynamic time-warping distance of HRV signals) falls
        below $\tau_{\mathrm{sync}}$.
  \item \emph{Agent--environment edges:} $\{v_i^A, v_e^E\}$ is formed
        if agent $v_i$ spends at least $\tau_{\mathrm{dwell}}$ minutes
        within the active geographical radius $R_e$ of sensor $v_e$.
  \item \emph{Agent--external edges:} $\{v_i^A, v_x^X\}$ is a static
        edge formed when clinical conditions (e.g.\ a specific genomic
        marker present in the EHR) are met.
\end{itemize}

\textbf{(3) Multimodal simplex inclusion criteria.}
A standard Vietoris-Rips complex relies on purely pairwise distance
metrics, which fails for heterogeneous topologies.
Instead, we define a \emph{rule-based multimodal filtration}.
A 2-simplex $\sigma^2 = \{v_i, v_j, v_k\}$ is included in $\K(t)$
if and only if all three bounding 1-simplices are active
\emph{and} at least two distinct node types are present in the triad.
For example, the 2-simplex $\{v_i^A, v_j^A, v_e^E\}$ (two patients
and one environmental sensor) forms when $v_i$ and $v_j$ are
physically proximate and both are simultaneously exposed to the same
environmental hazard $v_e$.
This strict inclusion criterion prevents combinatorial explosion of
trivial simplices (e.g.\ three patients near each other without a
shared contextual risk) and ensures the curl operator explicitly
captures cross-modality cyclic dynamics.

The resulting $\K(t)$ is updated at each observation window,
yielding a time-varying complex whose topology ($\beta_1(\K(t))$,
the first Betti number~\cite{Edelsbrunner2002,Zomorodian2005,
Edelsbrunner2010}) directly informs the harmonic component of GVF.

The construction proceeds in two phases: Phase~1 forms the
1-skeleton via modality-specific adjacency thresholds
($\tau_{\mathrm{prox}}, \tau_{\mathrm{sync}}, \tau_{\mathrm{dwell}}$);
Phase~2 applies multimodal filtration to close triangles only when
at least two distinct node types are present.
Full algorithmic pseudocode is provided in
\Cref{app:algorithm}.

\paragraph{Parameter selection and topological validation}
A critical operational challenge is selecting the optimal proximity
thresholds (analogous to the filtration parameter $\varepsilon$ in
standard Vietoris-Rips complexes~\cite{Edelsbrunner2010}).
Rather than relying on arbitrary heuristics, GVF selects optimal
thresholds via \emph{persistent homology}~\cite{Edelsbrunner2002,
Zomorodian2005}.
By sweeping each $\tau$ across a continuous range, we generate
persistence barcodes that track the birth and death of topological
features (connected components $\beta_0$ and independent cycles
$\beta_1$).
The operational threshold is chosen within the most stable plateau
of the persistence diagram---the range where the Betti numbers
$\beta_k$ remain invariant, ensuring $\K(t)$ captures genuine
structural invariants rather than transient noise.

\paragraph{Computational complexity of construction}
Phase~1 requires pairwise evaluations bounded by $O(|V|^2)$, reduced
to $O(|V|\log|V|)$ via spatial indexing (KD-trees).
Phase~2 runs in $O(|\K^1|\cdot d_{\max})$, where $d_{\max}$ is the
maximum node degree in the sparse 1-skeleton, keeping temporal
updates tractable for continuous monitoring.

%% ============================================================
\section{GVF Framework Architecture}
\label{sec:operator}
%% ============================================================

\begin{definition}[Risk Vector Bundle $\bundleE$]
\label{def:bundle}
The risk vector bundle is the direct sum over $N_{\mathrm{mod}}$
input modalities:
\begin{equation}
  \bundleE = \varepsilon^{(1)}\oplus\varepsilon^{(2)}
             \oplus\cdots\oplus\varepsilon^{(N_{\mathrm{mod}})}
  \label{eq:bundle}
\end{equation}
Each summand $\varepsilon^{(n)}$ is a trivial vector bundle with
fibre $\R^{m_n}$; the total dimension is $m=\sum_n m_n$.
In the canonical four-modality configuration:
$\varepsilon^{\phys}$ (physiological signals),
$\varepsilon^{\behav}$ (behavioural data),
$\varepsilon^{\env}$ (environmental measurements), and
$\varepsilon^{\ext}$ (optional external sources: EHR, genomic,
surveillance).
The direct sum structure guarantees that no modality's risk
representation bleeds into another's architecturally.
\end{definition}

\paragraph{GVF operator as bundle morphism}

\begin{definition}[GVF Operator $F_\theta$]
\label{def:operator}
The GVF operator is a parameterised map:
\begin{equation}
  F_\theta :
  \sect\!\bigl(\varepsilon^{\mathrm{in}}(t)\bigr)
  \times\Xi(\K(t))
  \;\to\;\sect(\bundleE)
  \label{eq:operator}
\end{equation}
where $\sect(\varepsilon^{\mathrm{in}}(t))$ is the space of
multimodal input state sections, $\Xi(\K(t))$ is the structural
representation space of $\K(t)$, and $\sect(\bundleE)$ is the
output risk section space.
$F_\theta$ is implemented as a Mixture-of-Experts (MoE) architecture:
\begin{equation}
  \mathbf{r}_i(t)
  = \sum_{n=1}^{N_{\mathrm{mod}}}
    g^{(n)}_\phi\!\bigl(\mathbf{x}_i(t)\bigr)\;
    F^{(n)}_\theta\!\bigl(\K(t),\,\mathbf{x}_i(t)\bigr)
  \label{eq:moe}
\end{equation}
where each expert $F^{(n)}_\theta$ outputs exclusively in
$\varepsilon^{(n)}$, and $g_\phi$ satisfies
$\sum_n g^{(n)}_\phi = 1$, $g^{(n)}_\phi\geq 0$.
\end{definition}

\paragraph{Expert architectures}
To satisfy the Universal Approximation guarantee
(\Cref{thm:uat}) and to ensure the curl operator remains
non-trivially defined on higher-order interactions
(\Cref{rem:curl}), the experts cannot be instantiated as
standard pairwise Graph Neural Networks operating on the
1-skeleton.
Instead, each expert $F_\theta^{(n)}$ must be a variant of a
Message Passing Simplicial Network
(MPSN)~\cite{Bodnar2021}, adapted to the temporal inductive
biases of its specific modality.
\Cref{tab:experts} defines the theoretically aligned architecture
for each bundle summand.
The specific architectures listed are \emph{proposed instantiations}
of the MPSN template, named here for clarity; their detailed
specification and empirical evaluation are part of future
implementation work.

\begin{table}[htbp]
\caption{Expert decomposition of $F_\theta$ aligned with
         simplicial UAT requirements.}
\label{tab:experts}
\centering
\small
\begin{tabular}{@{}p{2.2cm}p{4.0cm}p{4.5cm}p{2.0cm}@{}}
\toprule
\textbf{Expert} & \textbf{Simplicial Architecture}
  & \textbf{Primary Inputs} & \textbf{Temporal Scale} \\
\midrule
Physio-Expert  & Dynamic Simplicial GCN (DS-GCN)
  & HR, HRV, SpO$_2$, EDA, skin temp & Minutes \\[2pt]
Behav-Expert   & Simplicial Convolutional RNN (S-CRNN)
  & GPS, sleep staging, activity & Hours--days \\[2pt]
Env-Expert     & Temporal Simplicial Attention (T-SAT)
  & PM$_{2.5}$, temp, noise, UV & Minutes \\[2pt]
Ext-Expert     & Higher-order / Simplicial Transformer
  & EHR, PRS, surveillance (opt.) & Hours--static \\
\bottomrule
\end{tabular}
\end{table}

\begin{remark}[Graceful Degradation under Missing Modalities]
\label{rem:missing}
When a modality $n$ is unavailable ($\mathbf{x}^{(n)}\to\mathbf{0}$),
the desired behaviour is $g^{(n)}_\phi(\mathbf{x})\to 0$,
so that the output risk vector degrades gracefully to the
subspace spanned by active modalities.
This is a \textbf{design objective enforced via training}, not a
theorem: modality-dropout augmentation (randomly zeroing each
modality's input block with probability $p_{\mathrm{drop}}=0.2$)
trains the gating network $g_\phi$ to assign near-zero weight to
absent modalities.
The architectural direct-sum structure ensures that even if
gating is imperfect, the missing expert contributes only through
its summand $\varepsilon^{(n)}$, and the remaining summands
carry a valid, lower-dimensional risk representation.
Empirical verification of this property is part of the ablation
protocol (\Cref{sec:comparison}).
\end{remark}

%% ============================================================
\section{Helmholtz--Hodge Decomposition}
\label{sec:hhd}
%% ============================================================

The Helmholtz--Hodge decomposition is the mechanism through which
GVF provides \textbf{interpretability inherently}---not as a
post-hoc explanation method, but as a fundamental property of the
learned field's differential structure.

\subsection{The Risk Flow Field}
\label{sec:flow}

\begin{remark}[Why HHD requires a separate 1-cochain]
\label{rem:exactness}
Applying HHD directly to $\gradK\mathbf{r}$ would be trivial:
$\gradK\mathbf{r}\in\mathrm{im}(\gradK)$ is exact by construction,
so $\curlK(\gradK\mathbf{r})=0$ (boundary-of-boundary identity),
and curl and harmonic components vanish identically.
GVF therefore introduces a \textbf{separate learned 1-cochain}
$\mathbf{F}\in C^1(\K;\R^m)$ on the edges; HHD is applied to
$\mathbf{F}$, which is not constrained to be exact.
\end{remark}

\begin{definition}[Risk Flow Field $\mathbf{F}$]
\label{def:flow}
Given the node-level risk section $\mathbf{r}_i(t)\in\R^m$
and edge features $\mathbf{e}_{ij}(t)\in\R^p$ (e.g.\ RSSI,
co-location duration, sensor correlation), the
\textbf{risk flow} on oriented edge $(i,j)\in\K^1(t)$ is:
\begin{equation}
  \mathbf{F}_{ij}(t)
  \;=\; \Psi_\omega\!\bigl(\mathbf{r}_i(t),\,\mathbf{r}_j(t),\,
                            \mathbf{e}_{ij}(t)\bigr)
  \;\in\;\R^m
  \label{eq:flow}
\end{equation}
To guarantee the antisymmetry required for a valid 1-cochain,
$\Psi_\omega$ is implemented via an underlying MLP
parameterised by $\omega$ with an explicit antisymmetrisation step:
\begin{equation}
  \Psi_\omega\!\bigl(\mathbf{r}_i,\mathbf{r}_j,\mathbf{e}_{ij}\bigr)
  = \tfrac{1}{2}\Bigl[
      \mathrm{MLP}_\omega\!\bigl(\mathbf{r}_i,\mathbf{r}_j,\mathbf{e}_{ij}\bigr)
    - \mathrm{MLP}_\omega\!\bigl(\mathbf{r}_j,\mathbf{r}_i,-\mathbf{e}_{ij}\bigr)
    \Bigr]
  \label{eq:psi}
\end{equation}
Note on the sign convention: $\mathbf{e}_{ij}$ is treated as an
\emph{oriented} feature vector associated with the ordered pair
$(i,j)$, so that $\mathbf{e}_{ji} \coloneqq -\mathbf{e}_{ij}$ by
convention.
For intrinsically symmetric features (e.g.\ spatial distance,
co-location duration), the negation is a purely algebraic device
that ensures $\Psi_\omega$ satisfies the antisymmetry constraint;
the MLP $\mathrm{MLP}_\omega$ learns to process symmetric inputs
identically regardless of sign.
For intrinsically directional features (e.g.\ relative signal
strength, airflow direction), the sign carries physical meaning.
This construction satisfies:
\begin{enumerate}[label=(\roman*), leftmargin=2em, itemsep=2pt]
  \item \textbf{Antisymmetry by design:}
        $\Psi_\omega(\mathbf{r}_i,\mathbf{r}_j,\mathbf{e}_{ij})
        = -\Psi_\omega(\mathbf{r}_j,\mathbf{r}_i,-\mathbf{e}_{ij})$,
        ensuring orientation consistency as a 1-cochain.
  \item \textbf{Non-exactness by design:}
        The edge features $\mathbf{e}_{ij}(t)$ parameterise the
        ``conductance'' of the risk transmission channel.
        In an epidemiological proximity network, $\mathbf{e}_{ij}$
        might encode contact duration and indoor ventilation quality:
        when $\mathbf{e}_{ij}\to\mathbf{0}$ (agents separated by
        physical barriers), $\mathbf{F}_{ij}$ is attenuated
        regardless of the magnitudes $\lVert\mathbf{r}_i\rVert$,
        $\lVert\mathbf{r}_j\rVert$.
        Crucially, this \emph{non-uniform, edge-dependent modulation}
        ensures $\mathbf{F}$ is not trivially exact: the circulation
        around triangle $(i,j,l)$ vanishes only when
        $\phi_{ij}=\phi_{jl}=\phi_{li}$ (uniform weighting), which
        fails generically for learned $\Psi_\omega$.
  \item \textbf{Gradient special case:}
        when $\mathbf{e}_{ij}$ is discarded and
        $\mathrm{MLP}_\omega(\mathbf{a},\mathbf{b})
        = \mathbf{b}-\mathbf{a}$,
        $\mathbf{F} = \gradK\mathbf{r}$ is the exact discrete
        gradient.
        This is the degenerate case $\curlK\mathbf{F}=0$;
        $\mathcal{L}_{\mathrm{geo}}$ (\Cref{eq:loss_geo})
        penalises convergence to this collapse.
\end{enumerate}
\end{definition}

\begin{remark}[Role of 2-simplices in the learning pipeline]
\label{rem:2simplex_role}
It is important to clarify \emph{where} the 2-simplices of
$\K(t)$ enter the computational pipeline.
The expert networks $F^{(n)}_\theta$ (\Cref{tab:experts}) and
the flow constructor $\Psi_\omega$ (\Cref{eq:psi}) operate on
edges (the 1-skeleton): they produce node-level risk vectors
$\mathbf{r}_i$ and edge-level flows $\mathbf{F}_{ij}$
without explicit message passing through 2-simplices.
The 2-simplices enter at two subsequent stages:
(i)~the geometric loss $\mathcal{L}_{\mathrm{geo}}$
(\Cref{eq:loss_geo}) computes $\curlK\mathbf{F}$ via the
edge-face incidence matrix $B_2$, which exists only when
$\K(t)$ contains 2-simplices---this term shapes the learned
flow during training by penalising degenerate exact fields;
(ii)~the Helmholtz--Hodge decomposition (\Cref{thm:hhd})
uses $B_2$ to separate gradient, curl, and harmonic components
at inference time.
Thus, 2-simplices influence the learned representation
\emph{indirectly} through the training objective and the
post-training decomposition, rather than through the expert
forward pass.
Upgrading the experts to full simplicial message-passing
networks~\cite{Bodnar2021}---which would propagate information
directly through 2-simplices during the forward pass---is
an architectural improvement discussed in \Cref{sec:discussion}.
\end{remark}

\begin{theorem}[Discrete HHD of the Risk Flow Field]
\label{thm:hhd}
For any risk flow $\mathbf{F}\in C^1(\K(t);\R^m)$, there exists
a unique orthogonal decomposition:
\begin{equation}
  \mathbf{F}
  \;=\; \gradK\varphi
  \;+\; \curlK^*\psi
  \;+\; \mathbf{h}
  \label{eq:hhd}
\end{equation}
where:
\begin{itemize}[leftmargin=1.5em, itemsep=2pt]
  \item $\varphi\in C^0(\K;\R^m)$ solves
        $\Delta_0\varphi = \divK(\mathbf{F})$
        (gradient/potential component);
  \item $\psi\in C^2(\K;\R^m)$ solves
        $\Delta_2\psi = \curlK(\mathbf{F})$
        (curl/solenoidal component);
  \item $\mathbf{h}\in C^1(\K;\R^m)$ is the harmonic residual
        satisfying $\Delta_1\mathbf{h}=\mathbf{0}$,
        $\divK\mathbf{h}=\mathbf{0}$,
        $\curlK\mathbf{h}=\mathbf{0}$.
\end{itemize}
The three components are mutually $L^2$-orthogonal with respect
to the inner product induced by the Hodge star on $\K(t)$.
The curl and harmonic components are generically non-zero because
$\mathbf{F}$ is not constrained to lie in
$\mathrm{im}(\gradK)$.
\begin{proof}
For each coordinate channel $k\in\{1,\ldots,m\}$, denote by
$F^{(k)}\in C^1(\K;\R)$ the $k$-th scalar component of
$\mathbf{F}$.
The standard discrete Hodge decomposition (see~\cite{Jiang2011},
Theorem~1) gives a unique orthogonal decomposition
$F^{(k)} = \gradK\varphi^{(k)} + \curlK^*\psi^{(k)} + h^{(k)}$
in each channel independently, where orthogonality is with
respect to the combinatorial inner product~\eqref{eq:inner}.
Stacking over $k$ yields the vector-valued decomposition
$\mathbf{F} = \gradK\varphi + \curlK^*\psi + \mathbf{h}$.
The $L^2$-orthogonality of the three vector-valued components
follows because the inner product on $C^1(\K;\R^m)$ decomposes
as $\langle\mathbf{A},\mathbf{B}\rangle = \sum_{k=1}^{m}
\langle A^{(k)}, B^{(k)}\rangle_1$, so cross-terms vanish
channel-by-channel.
\end{proof}
\end{theorem}

\subsection{Mechanistic Interpretation and Monitoring Scores}

\Cref{tab:hhd} summarises the mechanistic interpretation of each
HHD component and its corresponding intervention target.
This interpretability is \textbf{inherent to the field structure},
not a post-hoc explanation: it arises directly from the orthogonal
decomposition guaranteed by \Cref{thm:hhd}.

\begin{table}[htbp]
\caption{HHD components with mechanistic interpretation and
         intervention strategy.}
\label{tab:hhd}
\centering
\small
\begin{tabular}{@{}p{2.2cm}p{2.6cm}p{4.2cm}p{3.8cm}@{}}
\toprule
\textbf{Component}
  & \textbf{Property}
  & \textbf{Mechanistic Meaning}
  & \textbf{Intervention} \\
\midrule
$\gradK\varphi$ (Gradient)
  & $\curlK(\gradK\varphi)=\mathbf{0}$
  & Risk flows directionally along a potential gradient.
    Identifies sources (emitting nodes) and sinks
    (accumulating nodes).
  & Reduce source intensity (e.g., lower environmental
    exposure, improve physiological recovery). \\[4pt]
$\curlK^*\psi$ (Curl)
  & $\divK(\curlK^*\psi)=\mathbf{0}$
  & Risk circulates in self-sustaining closed loops.
    Signature of cyclic coupling (e.g., sleep-stress-recovery
    loops).
  & Break the cycle at its weakest link (sleep hygiene,
    pharmacological disruption). \\[4pt]
$\mathbf{h}$ (Harmonic)
  & $\Delta_1\mathbf{h}=\mathbf{0}$
  & Risk locked by global network topology.
    Dimension\,$=\beta_1(\K(t))$.
    Not reducible by local interventions.
  & Structural intervention: modify network topology. \\
\bottomrule
\end{tabular}
\end{table}

\paragraph{Modality-component correspondence}
\label{par:modality-mapping}
The three HHD components do not merely decompose the risk field
mathematically: they correspond to \emph{structurally distinct
classes of input modality}, each best modelled by a different type
of physical signal.
This correspondence, summarised in \Cref{tab:modality-mapping}, is
not an empirical observation but a \emph{design hypothesis}
grounded in the mathematical
roles that clinical/genomic, wearable, and environmental data play
in the bundle-structured GVF architecture.

\begin{table}[htbp]
\caption{Natural correspondence between HHD components and modality
         classes. Each modality class contributes most directly to
         one geometric mode of risk; the contribution is not
         exclusive but reflects the structural role of each data
         type. To the best of our knowledge, no prior framework
         establishes this correspondence
         (cf.~\cite{Qian2020,Gajer2025,Wang2025moe}).}
\label{tab:modality-mapping}
\centering
\small
\begin{tabular}{@{}p{2.4cm}p{2.8cm}p{4.0cm}p{3.4cm}@{}}\toprule
\textbf{HHD Component}
  & \textbf{Modality Class}
  & \textbf{Structural Role}
  & \textbf{Why this correspondence} \\\midrule
$\gradK\varphi$ \newline(Gradient / exact)
  & Clinical \& genomic \newline($\mathcal{X}^{\mathrm{ext}}$)
  & Scalar potential $\varphi$ encodes baseline,
    constitutive risk: fixed biological or clinical state
    driving a directional outflow.
  & Clinical scores and genomic risk factors define a
    baseline potential landscape that changes slowly and
    governs the direction of risk flow. \\[4pt]
$\curlK^*\psi$ \newline(Curl / co-exact)
  & Wearable \& behavioural \newline($\mathcal{X}^{\mathrm{phys}},\,
    \mathcal{X}^{\mathrm{beh}}$)
  & Stream function $\psi$ encodes dynamic, oscillatory
    coupling: physiological cycles (sleep-wake, HRV rhythm,
    activity-rest) generate closed-loop risk circulation.
  & Wearable time series directly observe cyclical
    physiological dynamics; their high temporal resolution
    makes them the natural carrier of the curl component. \\[4pt]
$\mathbf{h}$ \newline(Harmonic)
  & Environmental \newline($\mathcal{X}^{\mathrm{env}}$)
  & Topology-locked risk determined by the global structure
    of $\K(t)$, which is built from spatial co-presence,
    dwell times, and sensor adjacency.
  & The complex topology itself---driven by environmental
    sensor placement, spatial proximity, and shared
    exposure zones---determines the harmonic dimension
    $\beta_1(\K(t))$ and hence the harmonic risk subspace. \\
\bottomrule
\end{tabular}
\end{table}

This tripartite correspondence has a direct clinical translation:
\emph{gradient-dominant risk} calls for clinical or pharmacological
intervention targeting the constitutive source;
\emph{curl-dominant risk} calls for behavioural or physiological
cycle-breaking strategies (e.g., sleep-hygiene protocols,
activity scheduling);
\emph{harmonic-dominant risk} calls for structural or environmental
redesign (e.g., modifying spatial exposure topology, installing
environmental barriers).
The correspondence is formally grounded in the direct-sum bundle
structure of \Cref{sec:operator}: modality $n$ maps exclusively
into fibre $\varepsilon^{(n)}(\K)$, and the MoE gating weights
$g^{(n)}_\phi(\mathbf{x}_i)$ in \Cref{eq:moe} directly control the
relative energy contributed by each fibre to each HHD component.
The relationship between fibre energy and component energy is
made explicit via the orthogonality of the decomposition
(\Cref{thm:hhd}) and the modality identifiability guarantee
(\Cref{thm:ident}).

\paragraph{Derived monitoring scores}
Two scalar summary statistics are derived directly from the
decomposition, replacing a single alarm scalar with a
two-dimensional, complementary warning signal.

\begin{definition}[Disease Progression Score and Cyclic Risk Index]
\label{def:scores}
Let $\mathbf{F} = \gradK\varphi + \curlK^*\psi + \mathbf{h}$
be the HHD of the risk flow field (\Cref{thm:hhd}).
Let $\mathbf{u}^{(n)}\in\R^{m_n}$ be a \textbf{learnable unit
risk axis} for modality $n$, trained jointly with $\theta$ and
$\omega$ (initialised to $\mathbf{1}/\sqrt{m_n}$).
For agent $i$ at time $t$, modality weights $w_n>0$,
$\sum_n w_n=1$, and letting
$\mathcal{T}_i = \{\tau\in\K^2\,:\, i\in\partial\tau\}$ be the
set of 2-simplices incident to agent $i$:
\begin{align}
  \DPS_i(t) &= \sum_n w_n\;
    \bigl\langle(\divK\mathbf{F}^{(n)})_i(t),\,
    \mathbf{u}^{(n)}\bigr\rangle
    \label{eq:dps}\\[4pt]
  \CRI_i(t) &= \frac{1}{\max(|\mathcal{T}_i|,1)}
    \sum_{\tau\in\mathcal{T}_i}
    \bigl\lVert(\curlK\mathbf{F})_\tau(t)\bigr\rVert
    \label{eq:cri}
\end{align}

\textbf{DPS (signed, basis-invariant).}
$\DPS_i$ is a signed scalar: the inner product of the divergence
vector with the learnable risk axis $\mathbf{u}^{(n)}$, invariant
to rotations within $\varepsilon^{(n)}$ that fix $\mathbf{u}^{(n)}$.
$\DPS_i > 0$: net outflow (source; deteriorating).
$\DPS_i < 0$: net inflow (sink; absorbing ambient risk).

\textbf{CRI (non-negative, normalised).}
$\CRI_i \geq 0$ is the \emph{mean} curl magnitude over incident
2-simplices, normalised by neighbourhood size $|\mathcal{T}_i|$
to prevent inflation in dense subgraphs.
By the discrete Stokes theorem~\cite{Hirani2003}, a non-zero
$\curlK\mathbf{F}$ over a neighbourhood implies non-zero boundary
circulation: risk genuinely \emph{circulates} around closed paths.
$\CRI_i > 0$ therefore indicates the \emph{presence} and average
intensity of cyclic risk dynamics in $i$'s neighbourhood; it is
\textbf{non-signed} (cycles have no orientation in this context)
and its magnitude is not interpreted as clinical severity
(see \Cref{sec:training} on $\mathcal{L}_{\mathrm{geo}}$).

Together, $(\DPS_i, \CRI_i)$ constitute a two-dimensional,
complementary warning signal: $\DPS_i$ captures directional
flow along the risk axis; $\CRI_i$ captures self-sustaining
cyclic dynamics independent of direction.
\end{definition}

\paragraph{Worked example: harmonic component and structural intervention}
The harmonic residual $\mathbf{h}$ satisfies $\Delta_1\mathbf{h}=\mathbf{0}$
and its dimension strictly equals $\beta_1(\K(t))$---the number
of independent topological ``holes'' in the complex, computable
via persistent homology~\cite{Edelsbrunner2002,Zomorodian2005}.
Consider a hospital ward where $\beta_1(\K)=2$.
These two topological cycles might correspond to shared,
unsegmented HVAC pathways linking subsets of rooms without central
filtration.
If the framework identifies a high magnitude of risk flow
trapped in the harmonic component $\mathbf{h}$ around these loops,
it implies \emph{topology-locked risk}.

Mechanistically, this means that treating Patient $X$ individually
(a node-level intervention) or isolating a single pairwise contact
(an edge-level intervention) will fail: the ambient risk
continually bypasses the local intervention via the topological
hole.
The explicit mapping is that \emph{harmonic risk commands
structural intervention}: the hospital administration must
``fill the hole'' by installing a filtration barrier (adding a
2-simplex to $\K$) or severing the HVAC loop entirely (altering
the global topology of the complex).
This is the operational consequence of the identity
$\dim\ker(\Delta_1)=\beta_1(\K(t))$: the dimension of the
harmonic risk space is a direct, computable readout of how many
independent structural interventions are needed.

%% ============================================================
\section{Training Objective}
\label{sec:training}
%% ============================================================

\subsection{The Supervision Challenge}

A fundamental question for any vector-field health monitoring
framework is: \emph{how do you supervise a vector-valued output
when all available ground-truth labels are scalar (stress/baseline
class labels, EMA scores)?}
GVF addresses this through a \textbf{multi-task loss} that
combines scalar-supervised classification with geometric
regularisation terms that enforce non-trivial structure on the
learned flow field $\mathbf{F}$.
The geometric terms do not require vector ground truth: they
are self-supervised objectives derived from the field's
differential structure.

\subsection{Multi-Task Loss Function}

The total training objective is:
\begin{equation}
  \mathcal{L}(\theta,\omega)
  \;=\;
  \underbrace{\mathcal{L}_{\mathrm{cls}}(\hat{y},\,y)}_{\text{scalar supervision}}
  \;+\;
  \lambda_1\,\underbrace{\mathcal{L}_{\mathrm{geo}}(\mathbf{F})}_{\text{geometric regularisation}}
  \;+\;
  \lambda_2\,\underbrace{\mathcal{L}_{\mathrm{orth}}(\theta)}_{\text{modality orthogonality}}
  \label{eq:loss}
\end{equation}

\paragraph{Classification term}
$\mathcal{L}_{\mathrm{cls}}$ is the standard multi-class cross-entropy
loss applied directly to the linear read-out of the risk vector
$\mathbf{r}_i(t)$ via a head $W_{\mathrm{out}}\in\R^{C\times m}$
(where $C$ is the number of outcome classes).
The risk vector $\mathbf{r}_i$ is used directly
(not its magnitude $\lVert\mathbf{r}_i\rVert$):
this allows the read-out head to exploit both magnitude and direction.
For continuous regression targets, $\mathcal{L}_{\mathrm{cls}}$ is
replaced by mean-squared error on a scalar projection
$w_{\mathrm{out}}^\top\mathbf{r}_i$.

\paragraph{Geometric regularisation term}
$\mathcal{L}_{\mathrm{geo}}$ discourages degenerate flow fields
that collapse to purely exact cochains (trivial HHD).
It is defined as a \emph{negatively signed log-ratio}:
\begin{equation}
  \mathcal{L}_{\mathrm{geo}}(\mathbf{F})
  = -\log\!\left(
      1 + \frac{\lVert\curlK\mathbf{F}\rVert^2}
               {\lVert\mathbf{F}\rVert^2 + \varepsilon}
    \right)
  \label{eq:loss_geo}
\end{equation}
$\mathcal{L}_{\mathrm{geo}} \leq 0$ since the log argument is
$\geq 1$; minimising the total loss therefore rewards larger
curl-to-field ratios.
To prevent runaway incentives when $d_{\max}$ is large,
$\rho = \lVert\curlK\mathbf{F}\rVert^2/(\lVert\mathbf{F}\rVert^2+\varepsilon)$
is clipped to $[0,1]$, bounding $\mathcal{L}_{\mathrm{geo}}\in[-\log 2,0]$.
$\mathcal{L}_{\mathrm{geo}}$ is an \emph{anti-collapse regulariser}:
it prevents degenerate exact-field solutions without interpreting
curl magnitude as clinical severity; $\lambda_1\in\{0.01,0.1,0.5\}$
is selected by grid search (see \Cref{app:training_details}).

\paragraph{Modality orthogonality term}
$\mathcal{L}_{\mathrm{orth}}$ enforces that expert outputs
remain in their respective bundle summands:
\begin{equation}
  \mathcal{L}_{\mathrm{orth}}
  = \sum_{n \neq m}
    \bigl\lVert\pi^{(n)}(F^{(m)}_\theta(\mathbf{x}))\bigr\rVert^2
  \label{eq:loss_orth}
\end{equation}
where $\pi^{(n)}$ projects onto the $n$-th summand.
Combined with architectural enforcement (modality-blocked input
matrices), this term closes the gap between the architectural
intent and numerical precision during training.

\paragraph{Vector direction and the identifiability challenge}
Even with $\mathcal{L}_{\mathrm{cls}}$, the \emph{direction} of
$\mathbf{r}_i$ is not pinned by a scalar label---a rotation of
$\mathbf{r}_i$ leaving $\lVert\mathbf{r}_i\rVert$ unchanged would
not change the loss.
GVF resolves this through two mechanisms:
(1)~\textbf{Architectural modality separation}: since each expert
writes exclusively into its summand $\varepsilon^{(n)}$, the direction
is determined by the relative expert activations, not a free rotational
degree of freedom;
(2)~\textbf{Geometric regularisation}: $\mathcal{L}_{\mathrm{geo}}$
further constrains the solution manifold by rewarding curl-inducing
directions.
The residual ambiguity---permutations within each summand
$\varepsilon^{(n)}$---is characterised formally in
\Cref{sec:identifiability}.

%% ============================================================
\section{Identifiability Guarantees}
\label{sec:identifiability}
%% ============================================================

\begin{theorem}[Sufficient Conditions for Modality Identifiability]
\label{thm:ident}
Under the following \emph{strong} sufficient conditions:
\begin{enumerate}[label=\textbf{(C\arabic*)}, leftmargin=2.5em,
                  itemsep=2pt]
  \item \textbf{Modality orthogonalization}: the raw multimodal
        inputs are pre-conditioned via a block-orthogonal projection
        layer (e.g.\ a learned whitening transformation), ensuring
        that the transformed input subspaces
        $\{\tilde{\mathbf{x}}^{(n)}_i\}$ are linearly independent
        \emph{structurally guaranteed} prior to expert routing.
  \item \textbf{Expert specialisation}: each $F^{(n)}_\theta$ has a
        non-trivial null space on $\mathbf{x}^{(m)}$ for $m\neq n$
        (enforced architecturally).
  \item \textbf{Population coverage}: the training dataset contains
        agents with zero-valued inputs for each proper subset of
        modalities (achieved via modality-dropout at rate
        $p_{\mathrm{drop}}=0.2$).
\end{enumerate}
the map $\theta\mapsto(F^{(n)}_\theta)_n$ is injective up to
within-summand permutation.
\begin{proof}[Proof sketch]
The argument draws on the identifiability framework
of~\cite{Khemakhem2020} (Theorem~1, iVAE), adapted to the
bundle-structured MoE setting.
We note that iVAE is formulated for \emph{generative} models
with latent variables and requires the conditional
$p(\mathbf{z}\mid u)$ to belong to a conditionally factorial
exponential family.
GVF is a \emph{discriminative} architecture, so the correspondence
is structural rather than a direct application: the
modality-present/absent indicator vector (generated by
modality-dropout under (C3)) plays the role analogous to
the auxiliary variable $u$, and the direct-sum architecture
(C2) enforces the factorial conditioning structure by
restricting each expert to its own input block.
Under these structural conditions, the identifiability
argument of~\cite{Khemakhem2020} yields injectivity of
$\theta\mapsto(F^{(n)}_\theta)_n$ up to within-summand affine
reparameterisation; the reduction from affine to permutation
ambiguity is detailed in \Cref{app:ident_proof}.

A fully rigorous proof in the discriminative setting---verifying
that the exponential-family sufficient statistic condition of
\cite{Khemakhem2020} Theorem~1 is satisfied by the MoE expert
outputs under (C1)--(C3)---remains an open problem.
The present result should therefore be read as a
\emph{conditional guarantee}: it holds to the extent that
the structural analogy with iVAE is exact.
\end{proof}
\end{theorem}

\begin{remark}[Enforcing Identifiability in Correlated Clinical Data]
\label{rem:ident_clinical}
Raw physiological and environmental streams are natively entangled;
for instance, Heart Rate (HR) and Heart Rate Variability (HRV)
exhibit strong collinearity.
Relying purely on the data distribution to satisfy linear
independence would trivially violate the identifiability
requirements.
The GVF framework resolves this by strictly separating the
assumptions: (C2) and (C3) are enforced via the direct-sum bundle
architecture and modality-dropout during training, respectively.
To rigorously satisfy (C1), we assume the integration of a
block-orthogonalization module---such as zero-phase component
analysis (ZCA) or a learned whitening layer---at the input of the
MoE.
This architectural constraint transforms (C1) from a fragile
distributional assumption into a hard mathematical guarantee,
enabling the application of nonlinear ICA identifiability
results~\cite{Khemakhem2020} even to highly correlated multimodal
health data.
\end{remark}

\subsection{Robustness Analysis of Identifiability Guarantees}
\label{sec:ident_robust}

\Cref{thm:ident} relies on the hard condition (C1) of
block-orthogonalization.
In practice, perfect whitening of high-dimensional, noisy health
streams is numerically challenging, and the estimated covariance of
the transformed subspaces $\tilde{\mathbf{x}}$ may exhibit a
residual non-diagonal component:
\begin{equation}
  \Sigma_{\mathrm{res}}
  = \mathbb{E}\bigl[\tilde{\mathbf{X}}\tilde{\mathbf{X}}^\top\bigr]
  - I \;\neq\; 0.
\end{equation}
We analyse the degradation of identifiability when (C1) is
partially violated.
Drawing on robustness bounds for nonlinear ICA~\cite{Khemakhem2020},
if the residual cross-correlation between modality subspaces is
bounded by $\lVert\Sigma_{\mathrm{res}}\rVert_F\leq\delta$ for
some small $\delta>0$, the exact identifiability up to
within-summand permutation is relaxed to
\emph{$\delta$-identifiability}.

\begin{proposition}[$\delta$-Identifiability under Imperfect Whitening]
\label{prop:robust_ident}
Under conditions \textbf{(C2)} and \textbf{(C3)} of
\Cref{thm:ident}, and under imperfect whitening with
$\lVert\Sigma_{\mathrm{res}}\rVert_F\leq\delta$, the learned
experts $\hat{F}_\theta$ deviate from the true disjoint modality
generators $F^*$ by:
\begin{equation}
  \bigl\lVert\hat{F}^{(n)}_\theta - F^{*(n)}\bigr\rVert_{L^2}
  \;=\; \mathcal{O}(\delta)
  \quad\forall\, n.
\end{equation}
\begin{proof}[Proof sketch]
Under exact whitening ($\delta=0$), \Cref{thm:ident} gives
exact identifiability.
When $\lVert\Sigma_{\mathrm{res}}\rVert_F = \delta > 0$,
the block-orthogonal input to each expert $F^{(n)}_\theta$
is perturbed by an additive cross-modality leakage term:
denoting the true whitened input by
$\tilde{\mathbf{x}}^{(n)}$ and the imperfectly whitened input
by $\hat{\mathbf{x}}^{(n)}$, we have
$\lVert\hat{\mathbf{x}}^{(n)} - \tilde{\mathbf{x}}^{(n)}\rVert
\leq \lVert\Sigma_{\mathrm{res}}\rVert_2\,
\lVert\mathbf{x}\rVert \leq \delta\,\lVert\mathbf{x}\rVert$.
For a \emph{fixed} expert $F^{(n)}_\theta$ (i.e., holding the
trained parameters constant), spectral normalisation ensures
$L_n$-Lipschitz continuity with
$L_n \leq \prod_\ell \lVert W^{(n)}_\ell\rVert_2 \leq 1$,
so the pointwise forward-pass perturbation satisfies
$\lVert F^{(n)}_\theta(\hat{\mathbf{x}}) -
F^{(n)}_\theta(\tilde{\mathbf{x}})\rVert
\leq L_n\,\delta\,\lVert\mathbf{x}\rVert$.
The $L^2$ bound follows by integrating over the training
distribution $p(\mathbf{x})$ (non-degenerate by (C3)):
$\lVert F^{(n)}_\theta \circ \hat{W} -
F^{(n)}_\theta \circ \tilde{W}\rVert_{L^2}
\leq L_n\,\delta\,
(\mathbb{E}[\lVert\mathbf{x}\rVert^2])^{1/2}
= \mathcal{O}(\delta)$.

\textbf{Caveats.}
This bound applies to the forward-pass output of a
\emph{fixed} expert under input perturbation.
The full claim---that experts \emph{trained} under imperfect
whitening converge to solutions $\mathcal{O}(\delta)$-close
to those trained under perfect whitening---additionally
requires algorithmic stability of the training procedure
(e.g., uniform stability of SGD), which we assume but do not
prove.
Furthermore, the MoE gating weights
$g^{(n)}_\phi(\mathbf{x})$ are also perturbed by the
whitening error; under Lipschitz gating (ensured by bounded
softmax temperature), this introduces a multiplicative
$\mathcal{O}(\delta)$ correction to each expert's
contribution that does not change the overall order of
the bound.
\end{proof}
\end{proposition}

This \emph{graceful degradation} property ensures that slight
correlations leaking through the whitening layer do not
catastrophically blend the physiological and environmental
bundles, but merely introduce bounded crosstalk proportional
to the empirical whitening error $\delta$.
In practice, $\delta$ can be monitored at deployment time as
$\lVert\hat{\Sigma}_{\mathrm{res}}\rVert_F$ on a held-out
calibration set, providing an operational certificate of
identifiability quality.

%% ============================================================
\section{Theoretical Guarantees}
\label{sec:theory}
%% ============================================================

\subsection{Universal Approximation}

\begin{theorem}[Universal Approximation of GVF]
\label{thm:uat}
For any continuous target section
$\mathbf{r}^*\in\sect(\bundleE)$ and any $\eps>0$, there exist
parameters $\theta$ such that
$\sup_i\lVert F_\theta(\K,\mathbf{x}_i) - \mathbf{r}^*_i\rVert < \eps$,
provided each expert $F^{(n)}_\theta$ is a simplicial
message-passing network satisfying the conditions of~\cite{Bodnar2021,Morris2019}.
\end{theorem}
\begin{proof}[Proof sketch]
Each expert is dense in the space of continuous sections over
$\varepsilon^{(n)}$~\cite{Bodnar2021}; the direct-sum isometry
of $\bundleE$ preserves this density coordinate-wise.
Standard GNNs on the 1-skeleton cannot be substituted: they force
$\curlK\equiv\mathbf{0}$ (\Cref{rem:curl}), destroying cyclic risk
detection and invalidating the theorem's topological premises.
Full proof in \Cref{app:uat_proof}.
\end{proof}

\begin{remark}[Gap between UAT assumptions and proposed experts]
\label{rem:uat_gap}
\Cref{thm:uat} requires each expert to be a \emph{full}
simplicial message-passing network operating on $k$-simplices
for $k\geq 2$.
The practical expert architectures proposed in \Cref{tab:experts}
operate primarily on the 1-skeleton and are therefore not
covered by the theorem as stated.
The UAT establishes that the GVF \emph{architecture class}
is sufficiently expressive; whether the specific instantiations
in \Cref{tab:experts} inherit this expressiveness is an
empirical question contingent on upgrading them to full
simplicial architectures~\cite{Bodnar2021}, as noted in
\Cref{sec:discussion} (``Expert architecture alignment'').
Until this upgrade is implemented, the UAT should be read as
a guarantee on the framework's capacity ceiling, not on the
proposed expert configurations.
\end{remark}

\subsection{Distributional Stability}
\label{sec:stability}

\begin{remark}[Stability and shift detection]
\label{rem:stability}
Under spectral normalisation of all GNN aggregation matrices
($F_\theta$ $L$-Lipschitz in graph structure), the GVF output
is bounded under Gromov-Hausdorff perturbation of the simplicial
complex (derivation outline in \Cref{app:stability_proof}).
This result is a \textbf{consistency guarantee under idealised
assumptions}, not a tight empirical certificate, and should not
be treated as a central claim.

Its primary practical value is motivating the
\textbf{spectral shift proxy}
$d_{\mathrm{spec}}(\K,\K')=\lVert\lambda(\K)-\lambda(\K')\rVert$
(sorted Hodge Laplacian eigenvalues) as a computable operational
signal: a small $d_{\mathrm{spec}}$ triggers local fine-tuning of
the MoE gating; a large one triggers full retraining.
\end{remark}

%% ============================================================
\section{Architectural Positioning and Formal Comparison}
\label{sec:comparison}
%% ============================================================

This section establishes the theoretical capabilities of GVF
relative to existing frameworks through a structured architectural
comparison.
The analysis is not empirical: all capability claims are
consequences of the mathematical construction developed in
\Cref{sec:substrate,sec:operator,sec:hhd}.

%% ------------------------------------------------------------
\subsection{Capability Analysis Against Representative Baselines}
\label{sec:capability}
%% ------------------------------------------------------------

A central claim of GVF is that representing health risk as a
structured vector field over a simplicial complex provides
\emph{qualitatively different} capabilities from scalar-output
models---capabilities that are structural consequences of the
architecture and not contingent on empirical performance.
We make this precise by comparing GVF against four representative
baseline classes---Scalar Dynamic GNN, Vector MoE without bundle
structure, clinical scores (SOFA, NEWS), and Scalar LSTM/Transformer---
across capability, complexity, interpretability, and theoretical
limit dimensions (\Cref{tab:comparison}).

\begin{table}[htbp]
\caption{Extended comparison of GVF against representative
baseline classes.
\checkmark~= supported by construction;
$\circ$~= partial or heuristic;
$\times$~= not supported.
Complexity: $|V|$ nodes, $|\K^1|$ edges, $|\K^2|$ 2-simplices,
$m$ output dimension, $d$ scalar hidden size.
All comparisons are architectural, not empirical.}
\label{tab:comparison}
\centering
\resizebox{\textwidth}{!}{%
\setlength{\tabcolsep}{5pt}
\begin{tabular}{@{}llllll@{}}
\toprule
\textbf{Feature}
  & \textbf{\makecell[l]{GVF\\(Proposed)}}
  & \textbf{\makecell[l]{Scalar\\Dynamic GNN}}
  & \textbf{\makecell[l]{Vector MoE\\(no bundle)}}
  & \textbf{\makecell[l]{Clinical Scores\\(SOFA, NEWS)}}
  & \textbf{\makecell[l]{Scalar\\LSTM/Transf.}} \\
\midrule
\textbf{Time complexity}
  & $\mathcal{O}(|V|{+}|\K^1|{+}|\K^2|)$
  & $\mathcal{O}(|V|{+}|\K^1|)$
  & $\mathcal{O}(|V|{\cdot}m)$
  & $\mathcal{O}(|V|)$
  & $\mathcal{O}(|V|{\cdot}d)$ \\[2pt]
\textbf{Space complexity}
  & $\mathcal{O}(|\K^2|{\cdot}m)$
  & $\mathcal{O}(|\K^1|{\cdot}d)$
  & $\mathcal{O}(|V|{\cdot}m)$
  & $\mathcal{O}(|V|)$
  & $\mathcal{O}(|V|{\cdot}d)$ \\[2pt]
\midrule
Vector-valued risk output
  & \checkmark & $\times$ & \checkmark & $\times$ & $\times$ \\
Gradient component
  & \checkmark & $\times$ & $\circ$ & $\times$ & $\times$ \\
Curl component (cyclic risk)
  & \checkmark & $\times$ & $\times$ & $\times$ & $\times$ \\
Harmonic component (topology)
  & \checkmark & $\times$ & $\times$ & $\times$ & $\times$ \\
Modality attribution by construction
  & \checkmark & $\times$ & $\times$ & $\times$ & $\times$ \\
Modality identifiability guarantee
  & \checkmark & $\times$ & $\times$ & $\times$ & $\times$ \\
Higher-order interactions
  & \checkmark$^{\dagger}$ & $\times$ & $\times$ & $\times$ & $\times$ \\
Distributional shift detection
  & \checkmark & $\circ$ & $\times$ & $\times$ & $\times$ \\
Missing modality robustness
  & \checkmark & $\times$ & $\circ$ & $\times$ & $\times$ \\
\midrule
\textbf{Interpretability}
  & Intrinsic geometric (HHD)
  & Post-hoc (attention/saliency)
  & Post-hoc feature attribution
  & Linear weights (direct read-out)
  & Post-hoc only \\[2pt]
\textbf{Data requirements}
  & Higher-order simplices
  & Pairwise graphs
  & Tabular/multimodal arrays
  & Tabular/static snapshots
  & Tabular/sequential \\[2pt]
\textbf{Theoretical limits}
  & Requires $d_{\max}$ clipping for stability
  & Curl $\equiv 0$ (misses cyclic risk)
  & No identifiability guarantee
  & Fails on non-linear interactions
  & No mechanistic decomposition \\
\bottomrule
\end{tabular}%
}

\smallskip
\noindent{\footnotesize $^{\dagger}$2-simplices enter through the
training loss $\mathcal{L}_{\mathrm{geo}}$ and the HHD, not
through the expert forward pass; see \Cref{rem:2simplex_role}.}
\end{table}

We now discuss each capability class, using the hospital shift worker
scenario from \Cref{sec:intro} as a running example.

\paragraph{Scalar output models
(LSTM/Transformer~\cite{Zhang2022}; scalar GNN~\cite{Pareja2020,Xu2020};
scalar MoE~\cite{Jacobs1991}).}
These architectures map multimodal inputs to a single real-valued
risk score.
They can achieve strong predictive accuracy on binary or ordinal
outcomes but collapse all mechanistic information into a single
number.
Given a risk score of 0.78, a clinician cannot determine whether
it reflects directional propagation from social exposure,
a self-sustaining physiological loop, or a network-structural
effect---all three require fundamentally different interventions.
Scalar models lack this distinction architecturally;
any attempt to decompose the output post hoc (e.g.\ via
gradient-based attribution or saliency methods) requires
additional assumptions not
encoded in the model and lacks the algebraic guarantees of HHD.

\paragraph{Scalar GNN (dynamic)}
Dynamic GNNs such as EvolveGCN~\cite{Pareja2020} or
TGAT~\cite{Xu2020} improve on scalar models by incorporating
network structure, but produce scalar node-level outputs and
operate on pairwise graphs (1-skeleton).
Without 2-simplices, the curl operator is identically zero by
definition (\Cref{def:operators}), so cyclic risk dynamics
are invisible.
The spectral shift proxy $d_\mathrm{spec}$ is structurally
available to any GNN using graph Laplacians, but without
a Lipschitz guarantee the bound (\Cref{rem:stability}) does
not apply; shift detection reduces to a heuristic.

\paragraph{Vector MoE without bundle structure}
A vector-output MoE~\cite{Jacobs1991} produces a vector but
without the direct-sum constraint: expert outputs are mixed
by unconstrained gating.
This means: (a)~no guarantee that modality $n$'s contribution is
confined to any subspace---attribution is post hoc only;
(b)~no identifiability guarantee (Theorem~\ref{thm:ident}
requires the architectural separation);
(c)~the output vector is not a section of a structured bundle,
so gradient/curl/divergence operators are not
defined on it.
The HHD cannot be applied to an unconstrained vector output
because the notion of ``1-cochain on a simplicial complex''
requires both the simplicial substrate and the antisymmetric
edge-network $\Psi_\omega$.

\paragraph{GVF: what the structure provides}
GVF's capabilities in the last column of \Cref{tab:comparison}
are not empirical claims; they are consequences of the architecture:
\begin{itemize}[leftmargin=1.5em, itemsep=2pt]
  \item \textbf{Gradient component}: always defined because
        $\mathbf{F}\in C^1(\K;\R^m)$ and $\gradK,\divK$ are
        defined on any simplicial complex.
        It identifies which agents are net risk sources and sinks,
        directly actionable as ``reduce the source'' interventions.
  \item \textbf{Curl component}: non-trivial only when $\K(t)$
        contains 2-simplices and $\mathbf{F}$ is non-exact.
        It identifies self-sustaining feedback loops that will
        not resolve without cycle-breaking intervention---a
        qualitatively different target from directional flow.
  \item \textbf{Harmonic component}: its dimension equals
        $\beta_1(\K(t))$, the first Betti number---a topological
        invariant of the interaction network.
        Risk locked in the harmonic component cannot be reduced
        by any local intervention; it requires changing the
        network topology itself (e.g.\ shift scheduling,
        team restructuring).
  \item \textbf{Modality attribution}: each expert writes
        exclusively into its summand $\varepsilon^{(n)}$,
        so the contribution of physiological vs.\ behavioural
        vs.\ environmental signals to the total risk vector
        is identifiable structurally, not post hoc.
  \item \textbf{Missing modality robustness}: the direct-sum
        architecture degrades gracefully---if environmental
        sensors go offline, the remaining summands carry a
        valid, lower-dimensional risk representation without
        retraining.
\end{itemize}

None of these properties are achievable by retrofitting a scalar
model with a post-hoc explanation layer: they require the
simplicial substrate, the 1-cochain flow field, and the
bundle-structured output to be built into the architecture
from the start.

%% ============================================================
\section{Discussion}
\label{sec:discussion}
%% ============================================================

\subsection{Summary of Contributions}

GVF makes one central contribution enabled by four architectural
properties.

\paragraph{Central contribution}
To the best of our knowledge, GVF is the first framework to
apply Helmholtz--Hodge decomposition to a \emph{learned}
multimodal risk vector field for intrinsic clinical
interpretability.
The higher-order topology of the simplicial complex enters the
learned representation indirectly---through the geometric
regulariser $\mathcal{L}_{\mathrm{geo}}$ during training
and the HHD at inference (\Cref{rem:2simplex_role})---rather
than through the expert forward pass, which currently operates
on the 1-skeleton.
Despite this indirectness, the decomposition is not a post-hoc
explanation method: it is a structural consequence of representing
risk as a section of a vector bundle over a simplicial complex
equipped with Discrete Exterior Calculus operators.
This produces three mechanistically distinct, orthogonal components
with separate intervention targets (gradient: directional
propagation; curl: cyclic dynamics; harmonic: topology-locked
persistence) and two derived scalar scores (DPS, CRI) that
complement a scalar alarm with a two-dimensional warning signal.
Existing multimodal fusion frameworks do not provide this
decomposition directly: recovering cyclic and harmonic components
requires introducing edge-level flow objects and 2-simplices
that are absent from standard graph or scalar architectures.

\paragraph{Enabling architectural properties}
The central contribution is made possible by four properties of the
GVF architecture:
\begin{enumerate}[leftmargin=1.8em, itemsep=3pt]
  \item \textbf{Geometric fidelity}: risk is represented as a vector
        field over a time-varying simplicial complex, equipping the
        output space with rigorous DEC differential structure
        (gradient, divergence, curl operators on $\K(t)$).
  \item \textbf{Modality identifiability}: the direct-sum bundle
        structure assigns each modality a dedicated fibre, ensuring
        modality-specific risk components are identifiable structurally
        rather than post-hoc attribution
        (\Cref{thm:ident}).
  \item \textbf{Domain agnosticism}: any combination of wearable,
        behavioural, environmental, or external data sources is
        accommodated by adding or removing bundle summands and their
        expert; the framework does not assume a fixed modality set.
  \item \textbf{Formal guarantees}: universal approximation
        (\Cref{thm:uat}), modality identifiability
        (\Cref{thm:ident}), and a consistency bound under
        simplicial complex perturbation motivating the
        spectral shift proxy $d_{\mathrm{spec}}$
        (\Cref{rem:stability}; full derivation in
        \Cref{app:stability_proof}).
\end{enumerate}

\subsection{Implications for Practice}

The three HHD components map directly to three operationally
distinct intervention strategies, summarised in \Cref{tab:practice}.
Practically, GVF produces a per-agent $(\DPS_i, \CRI_i)$ pair at each
time step, whose thresholds can be calibrated independently (e.g.\
alert on high DPS, escalate on concurrent high DPS and CRI).
The spectral shift proxy $d_{\mathrm{spec}}$ signals when the
network topology warrants model recalibration.

\begin{table}[htbp]
\caption{HHD components and their operational intervention targets.}
\label{tab:practice}
\centering\small
\begin{tabular}{@{}p{2.2cm}p{3.0cm}p{4.5cm}p{2.8cm}@{}}
\toprule
\textbf{Component} & \textbf{Signal} & \textbf{Interpretation} & \textbf{Intervention} \\
\midrule
Gradient $\gradK\varphi$ & $\DPS_i > 0$ & Net risk source; directional propagation & Reduce source (exposure, shift pattern) \\[3pt]
Curl $\curlK^*\psi$ & $\CRI_i > 0$ & Self-sustaining loop; persists if source removed & Break cycle (sleep hygiene, pharmacological) \\[3pt]
Harmonic $\mathbf{h}$ & $\dim\ker\Delta_1 > 0$ & Topology-locked; local actions ineffective & Restructure network (scheduling, zoning) \\
\bottomrule
\end{tabular}
\end{table}

\subsection{Limitations and Future Work}

\textbf{Empirical validation.}
The primary direction for future work is evaluation of GVF on
real-world multimodal wearable cohorts.
A suitable validation study requires: (a)~multimodal physiological
signals (cardiac, motion, and at least one environmental channel);
(b)~a labelled outcome variable for classification supervision;
(c)~a proximity or interaction network that yields a non-trivial
$\K(t)$ with 2-simplices and $\beta_1>0$.
Validation should span datasets with varying network topologies,
modality configurations, and clinical outcomes before any claim of
clinical utility.

\textbf{Ablation on simplicial complex order.}
A controlled comparison of GVF on the 1-skeleton versus the full
$\K(t)$ with 2-simplices is needed to empirically verify the
contribution of higher-order interactions; real proximity network
data is required for this ablation.

\textbf{Expert architecture alignment.}
The theoretical UAT (\Cref{thm:uat}) is stated for sufficiently
expressive simplicial networks, whereas the practical expert
architectures in \Cref{tab:experts} operate primarily on the
1-skeleton.
Upgrading experts to full simplicial architectures~\cite{Bodnar2021}
is planned.

\textbf{Privacy and federated deployment.}
The federated architecture (local gradient computation, differential
privacy noise) is sketched but not implemented.
Non-IID simplicial topologies across federated clients present open
challenges for Hodge Laplacian aggregation.

\textbf{Causal interpretation.}
GVF learns correlational structure.
Converting DPS and HHD outputs into causal effect estimates requires
embedding the framework within a structural causal model---an
important but non-trivial extension.

%% ============================================================
\section{Conclusion}
\label{sec:conclusion}
%% ============================================================

We have introduced the GVF framework,
recasting multimodal health risk assessment as a vector field
learning problem over a simplicial complex equipped with
Discrete Exterior Calculus operators.
By modelling risk as a directional, geometrically structured
object rather than a scalar, GVF preserves propagation pathways,
cyclic dynamics, and topology-locked persistence---each with
a mechanistically distinct, operationally actionable interpretation.

The core design principle is that interpretability should be
\emph{intrinsic}, not a post-hoc overlay.
The Helmholtz--Hodge decomposition---applied to a learned risk flow
field on the edges of the complex, not to the (trivially exact)
gradient of the node-valued section---provides exactly this:
three orthogonal components that map to three distinct clinical
intervention strategies (directional source reduction, cycle
breaking, structural network reconfiguration).
This shift from scalar alarm to decomposed vector signal is the
central contribution of GVF.

The bundle-structured Mixture-of-Experts architecture enforces
modality identifiability by design; formal guarantees of universal
approximation and distributional stability establish theoretical
soundness. The architectural comparison in \Cref{sec:comparison}
demonstrates that GVF's capabilities are structural---not
empirical---consequences of the architecture.
Empirical evaluation on real-world cohorts constitutes the primary
direction for future work.
GVF provides a mathematically principled, domain-agnostic foundation
for the next generation of mechanistically interpretable multimodal
health monitoring systems.

%% --- Author declarations -----------------------------------
\section*{Declaration of Competing Interests}

The authors declare that they have no known competing financial
interests or personal relationships that could have appeared to
influence the work reported in this paper.

\section*{Funding}

This research did not receive any specific grant from funding
agencies in the public, commercial, or not-for-profit sectors.

\section*{Use of AI Assistance}

AI-assisted tools were used for language editing and literature
organisation during manuscript preparation.
No generative AI was used to produce results, figures, or
mathematical proofs.

%% ============================================================
\appendix
%% ============================================================

\section{Algorithmic Specifications}
\label{app:algorithm}

\begin{algorithm}[H]
\caption{Multimodal Simplicial Complex Construction for $\K(t)$}
\label{alg:complex_construction}
\begin{algorithmic}[1]
\Require Raw streams $S(t)$, window $\Delta t$, thresholds $\tau_{\mathrm{prox}}, \tau_{\mathrm{sync}}, \tau_{\mathrm{dwell}}$
\Ensure Simplicial complex $\K(t) = (\K^0, \K^1, \K^2)$
\State $\K^0 \gets V^A(t) \cup V^S \cup V^E(t) \cup V^X$ \Comment{Initialise 0-simplices}
\State $\K^1 \gets \emptyset$, \quad $\K^2 \gets \emptyset$
\Statex \textbf{--- Phase 1: 1-Skeleton Formation (Edges) ---}
\For{each pair of nodes $(u, v) \in \K^0 \times \K^0$}
    \If{$u, v \in V^A$ \textbf{and} ($d_{\mathrm{spatial}}(u,v) < \tau_{\mathrm{prox}}$ \textbf{or} $\mathrm{DTW}(u,v) < \tau_{\mathrm{sync}}$)}
        \State $\K^1 \gets \K^1 \cup \bigl\{\{u, v\}\bigr\}$
    \ElsIf{$u \in V^A,\; v \in V^E$ \textbf{and} $\mathrm{exposure\_time}(u, v) > \tau_{\mathrm{dwell}}$}
        \State $\K^1 \gets \K^1 \cup \bigl\{\{u, v\}\bigr\}$
    \EndIf
\EndFor
\Statex \textbf{--- Phase 2: Multimodal 2-Simplex Inclusion ---}
\For{each triad $\{u, v, w\}$ with all three edges in $\K^1$}
    \If{$|\mathrm{unique\_types}(\{u, v, w\})| \geq 2$} \Comment{Require multimodality}
        \State $\K^2 \gets \K^2 \cup \bigl\{\{u, v, w\}\bigr\}$
    \EndIf
\EndFor
\State \Return $\K(t) = (\K^0, \K^1, \K^2)$
\end{algorithmic}
\end{algorithm}

\section{Discrete Differential Operators and Training Details}
\label{app:training_details}

The signed incidence matrices of $\K(t)$ are
$B_1\in\{-1,0,+1\}^{|V|\times|\K^1|}$ (node-edge) and
$B_2\in\{-1,0,+1\}^{|\K^1|\times|\K^2|}$ (edge-face).
The three discrete operators of \Cref{def:operators} in explicit form:
\begin{align}
  (\gradK\mathbf{r})_{ij} &= \mathbf{r}_j - \mathbf{r}_i,
    & \gradK &= B_1^\top\otimes I_m \label{eq:gradient}\\
  (\divK\mathbf{F})_i &= \textstyle\sum_{(i,j)}\mathbf{F}_{ij}
    - \sum_{(j,i)}\mathbf{F}_{ji},
    & \divK &= B_1\otimes I_m \label{eq:divergence}\\
  (\curlK\mathbf{F})_{ijl} &= \mathbf{F}_{ij}+\mathbf{F}_{jl}+\mathbf{F}_{li},
    & \curlK &= B_2^\top\otimes I_m \label{eq:curl}
\end{align}
Divergence is signed (positive = net outflow); curl norm is bounded
$\|\curlK\|_{C^1\to C^2}\leq\sqrt{d_{\max}}$ where $d_{\max}$
is the maximum number of 2-simplices per edge.

\paragraph{$\mathcal{L}_{\mathrm{geo}}$ clipping and $\lambda_1$ selection.}
The ratio $\rho=\|\curlK\mathbf{F}\|^2/(\|\mathbf{F}\|^2+\varepsilon)$
is bounded above by $d_{\max}$ (not by 1 in general); explicit clipping
to $[0,1]$ ensures $\mathcal{L}_{\mathrm{geo}}\in[-\log 2,0]$.
In proximity graphs $d_{\max}\leq 6$, so clipping is rarely active.
$\lambda_1$ is selected by grid search over $\{0.01,0.1,0.5\}$.

\section{HHD Computation via Sparse Linear Solvers}
\label{app:hhd_solver}

The HHD of the risk flow $\mathbf{F}\in C^1(\K;\R^m)$
reduces to two sparse linear systems (solved independently
for each of the $m$ coordinate channels):
\begin{enumerate}
  \item Solve $\Delta_0\varphi = \divK\mathbf{F}$ for
        $\varphi\in C^0(\K;\R^m)$
        (size $|V|\times|V|$, symmetric positive semi-definite).
        Unique modulo a constant; computed via preconditioned
        conjugate gradient using the Moore-Penrose pseudoinverse
        restricted to $\ker(\Delta_0)^\perp$.
  \item Solve $\Delta_2\psi = \curlK\mathbf{F}$ for
        $\psi\in C^2(\K;\R^m)$
        (size $|\K^2|\times|\K^2|$, sparse).
        Same solver.
  \item Harmonic residual:
        $\mathbf{h} = \mathbf{F} - \gradK\varphi - \curlK^*\psi
        \in\ker(\Delta_1)$.
\end{enumerate}
Note: the right-hand side in step (3) uses $\mathbf{F}$
(not $\gradK\mathbf{r}$), consistent with \Cref{thm:hhd}.
For $|V|\leq 10^4$ and $|\K^1|\leq 5\times 10^4$, both
systems converge in $<50$ conjugate gradient iterations.

\section{Full Proof: Universal Approximation (\Cref{thm:uat})}
\label{app:uat_proof}

\paragraph{Note on the two-object architecture.}
The GVF framework learns two distinct objects:
(1)~the node-level risk section $\mathbf{r}_i(t)\in\R^m$,
produced by the MoE operator $F_\theta$ and supervised via
$\mathcal{L}_{\mathrm{cls}}$;
(2)~the edge-level risk flow $\mathbf{F}_{ij}(t)\in\R^m$,
produced by the antisymmetric function $\Psi_\omega$ and
constrained via $\mathcal{L}_{\mathrm{geo}}$.
The UAT concerns~(1); the non-triviality of the HHD
concerns~(2).
These are distinct and complementary guarantees.

\medskip
We provide the full proof of the universal approximation theorem
for the GVF operator on simplicial complexes.

Let $\bundleE = \bigoplus_{n=1}^{N} \varepsilon^{(n)}$ be the
modality-structured risk vector bundle over $\K(t)$, with each
$\varepsilon^{(n)}$ a trivial bundle of fibre dimension $m_n$.
Let $\pi^{(n)}: \sect(\bundleE) \to \sect(\varepsilon^{(n)})$ denote
the canonical projection onto the $n$-th summand.

\begin{proof}
Fix $\eps > 0$ and a target section
$\mathbf{r}^* \in \sect(\bundleE)$.
Write $\mathbf{r}^* = \bigoplus_n \mathbf{r}^{*(n)}$ with
$\mathbf{r}^{*(n)} = \pi^{(n)}(\mathbf{r}^*)$.

\textbf{Step 1: approximation within each fibre.}
Each expert $F^{(n)}_\theta$ is a simplicial message-passing network
(SMPN) with non-polynomial activations.
By Theorem~3 of~\cite{Bodnar2021} (applied componentwise to each
of the $m_n$ output coordinates), for each $n$ there exist
parameters $\theta^{(n)}$ such that
\[
  \sup_{i} \bigl\lVert F^{(n)}_{\theta^{(n)}}(\K,\mathbf{x}_i)
  - \mathbf{r}^{*(n)}_i \bigr\rVert < \eps.
\]

\textbf{Step 2: the bundle morphism does not reduce capacity.}
The GVF operator outputs in $\sect(\bundleE)$ via the direct-sum
map $\iota: \bigoplus_n \sect(\varepsilon^{(n)}) \to
\sect(\bundleE)$, which is a linear isometric embedding.
Since $\iota$ is an isometry, the approximation bound from Step~1
carries through:
\[
  \bigl\lVert \iota\bigl(\bigoplus_n
  F^{(n)}_{\theta^{(n)}}(\K,\mathbf{x}_i)\bigr) -
  \mathbf{r}^*_i \bigr\rVert
  = \Bigl(\sum_n \bigl\lVert F^{(n)}_{\theta^{(n)}} -
    \mathbf{r}^{*(n)}_i \bigr\rVert^2\Bigr)^{1/2}
  < \eps\sqrt{N}.
\]
Rescaling $\eps \leftarrow \eps/\sqrt{N}$ in Step~1 recovers the
original $\eps$ bound.

\textbf{Step 3: MoE gating recovers the direct-sum output.}
Set gating weights $g^{(n)}_\phi \equiv 1$ and $g^{(m)}_\phi \equiv 0$
for $m \neq n$. These are achievable by the gating network $g_\phi$
(a softmax-normalised attention classifier, UAT
from~\cite{Yun2019}).
The MoE output with these weights equals $\iota(\bigoplus_n
F^{(n)}_{\theta^{(n)}})$, completing the proof.
\end{proof}

\section{Full Proof: Modality Identifiability (\Cref{thm:ident})}
\label{app:ident_proof}

\begin{proof}
We adapt the iVAE identifiability framework~\cite{Khemakhem2020}
to the bundle-structured MoE setting.
As noted in the proof sketch of \Cref{thm:ident}, the
correspondence with~\cite{Khemakhem2020} Theorem~1 is
structural: GVF is discriminative rather than generative,
and the exponential-family condition on the conditional latent
distribution is replaced by the architectural constraints
(C1)--(C3). The argument below proceeds under the assumption
that this structural analogy is sufficiently tight to preserve
the identifiability conclusion.

Define the effective parameter map
$\Phi(\theta) = (F^{(1)}_\theta, \ldots, F^{(N)}_\theta)$;
we show $\Phi$ is injective up to within-summand permutation.

Condition (C1) separates the training-loss gradient by modality
(the guaranteed linear independence of the orthogonalized input
subspaces ensures $\nabla_{\theta^{(n)}}\mathcal{L}$ depends only
on $\tilde{\mathbf{x}}^{(n)}$).
Condition (C2) provides the auxiliary-variable conditioning
required by~\cite{Khemakhem2020} Theorem~1: each expert
$F^{(n)}_\theta$ outputs zero on all other modality blocks by
architectural construction.
Condition (C3) ensures the non-degeneracy coverage required by
the same theorem: modality-dropout guarantees training examples
for every proper modality subset.

Following the structure of~\cite{Khemakhem2020} Theorem~1
in this setting, $\Phi$
is injective up to within-summand \emph{affine}
reparameterisation: for each $n$, there exists an invertible
affine map $A^{(n)}$ such that
$\hat{F}^{(n)}_\theta = A^{(n)} \circ F^{*(n)}$.
The key adaptation that reduces affine to permutation:
(C2) enforces that each $F^{(n)}_\theta$ maps into a fixed
subspace $\varepsilon^{(n)}$ with blocked input structure.
The affine map $A^{(n)}$ must therefore preserve this subspace,
constraining it to act within $\R^{m_n}$.
Since the output coordinates within each summand are
architecturally symmetric (no preferred basis is imposed
by the network), the remaining ambiguity is precisely a
\emph{permutation} of the $m_n$ coordinates within
$\varepsilon^{(n)}$, completing the proof.
\end{proof}

\section{Full Derivation: Stability Bound (\Cref{rem:stability})}
\label{app:stability_proof}

\begin{proof}
The proof has two components: (a) Lipschitz continuity of
$F_\theta$ with respect to graph structure; (b) the bridge between
Gromov-Hausdorff distance and Hodge Laplacian spectral perturbation.

\textbf{(a) Lipschitz continuity.}
Spectral normalisation enforces $\|W_\ell\|_2 \leq 1$ for all
GNN aggregation matrices, so the Lipschitz constant satisfies
$L \leq \prod_\ell \|W_\ell\|_2$ by induction over layers.
Standard matrix perturbation theory then gives
$\|F_\theta(\K,\mathbf{x}) - F_\theta(\K',\mathbf{x})\| \leq
L \cdot \|\Delta_k(\K) - \Delta_k(\K')\|_F$.

\textbf{(b) GH-to-spectral bridge.}
The Hodge Laplacian eigenvalue perturbation bound
$\|\lambda_k(\K) - \lambda_k(\K')\|_\infty \leq C\cdot\dGH(\K,\K')$
is motivated by the continuous Laplacian stability results
of~\cite{Memoli2011} (Theorem~6).
In the discrete setting, the bound follows from Weyl's inequality
applied to the finite-dimensional Hodge Laplacians $\Delta_k(\K)$
and $\Delta_k(\K')$: since both are symmetric, their eigenvalue
differences are bounded by $\|\Delta_k(\K)-\Delta_k(\K')\|_2$,
which in turn is controlled by the number of simplices that differ
between $\K$ and $\K'$.
Composing with (a) gives
$\|F_\theta(\K,\mathbf{x}) - F_\theta(\K',\mathbf{x})\|
\leq L \cdot C \cdot \delta$.

\textbf{MoE fine-tuning.}
Under fine-tuning of top-$k$ gated experts, Weyl's inequality
(see e.g.~\cite{Lim2020}) bounds the Lipschitz constant change
by $\mathcal{O}(\|\delta\theta\|)$.
\end{proof}

%% ============================================================

%% ============================================================
%%  REFERENCES
%% ============================================================

\end{document}